\documentclass[letterpaper]{article} 
\usepackage{aaai2026}  
\nocopyright
\usepackage{times}  
\usepackage{helvet}  
\usepackage{courier}  
\usepackage[hyphens]{url}  
\usepackage{graphicx} 
\usepackage{natbib}  
\usepackage{caption} 
\usepackage{algorithm}
\usepackage{algpseudocode}

\usepackage{newfloat}
\usepackage{listings}
\DeclareCaptionStyle{ruled}{labelfont=normalfont,labelsep=colon,strut=off} 
\floatstyle{ruled}
\newfloat{listing}{tb}{lst}{}
\floatname{listing}{Listing}
\title{No-Regret Strategy Solving in Imperfect-Information Games via Pre-Trained Embedding}

\author{
    Yanchang Fu\textsuperscript{\rm 1,2}\quad
    Shengda Liu\textsuperscript{\rm 2}\quad
    Pei Xu\textsuperscript{\rm 2}\quad
    Kaiqi Huang\textsuperscript{\rm 2}\thanks{Corresponding author.}
}
\affiliations{
    \textsuperscript{\rm 1}School of Artificial Intelligence, University of Chinese Academy of Sciences, Beijing 100049, P.R.China\\
    \textsuperscript{\rm 2}The Key Laboratory of Cognition and Decision Intelligence for Complex Systems, Institute of Automation, Chinese Academy of Sciences, Beijing 100190, China\\
    \{fuyanchang2020, shengda.liu, pei.xu\}@ia.ac.cn, kqhuang@nlpr.ia.ac.cn
}

\usepackage{bibentry}

\usepackage{amssymb} 
\usepackage{amsmath} 
\usepackage{subcaption}
\usepackage{amsthm} 
\usepackage{bm}
\usepackage{mathtools}
\usepackage{mathrsfs} 

 \usepackage{booktabs}
 \usepackage{xcolor}
 \usepackage{siunitx}
 \usepackage{makecell}

\theoremstyle{definition}
\newtheorem{definition}{Definition}[section]

\newtheorem{lemma}{Lemma}

\newtheorem{proposition}{Proposition}

\begin{document}

\maketitle

\begin{abstract}
High-quality information set abstraction remains a core challenge in solving large-scale imperfect-information extensive-form games (IIEFGs)--such as no-limit Texas Hold’em--where the finite nature of spatial resources hinders solving strategies for the full game. State-of-the-art AI methods rely on pre-trained discrete clustering for abstraction, yet their hard classification irreversibly discards critical information: specifically, the quantifiable subtle differences between information sets—vital for strategy solving—thus compromising the quality of such solving. Inspired by the word embedding paradigm in natural language processing, this paper proposes the Embedding CFR algorithm, a novel approach for solving strategies in IIEFGs within an embedding space. The algorithm pre-trains and embeds the features of individual information sets into an interconnected low-dimensional continuous space, where the resulting vectors more precisely capture both the distinctions and connections between information sets. Embedding CFR introduces a strategy-solving process driven by regret accumulation and strategy updates in this embedding space, with supporting theoretical analysis verifying its ability to reduce cumulative regret. Experiments on poker show that with the same spatial overhead, Embedding CFR achieves significantly faster exploitability convergence compared to cluster-based abstraction algorithms, confirming its effectiveness. Furthermore, to our knowledge, it is the first algorithm in poker AI that pre-trains information set abstractions via low-dimensional embedding for strategy solving.
\end{abstract}

\begin{links}
    \link{Code}{https://github.com/PhilEnchan/EmbeddingCFR}
\end{links}

\section{Introduction}

Solving large-scale imperfect-information games like no-limit Texas Hold'em remains a pivotal challenge in AI research. Systems including DeepStack~\cite{moravvcik2017deepstack}, Libratus~\cite{brown2018superhuman}, and Pluribus~\cite{brown2019superhuman} have achieved superhuman performance through game-theoretic equilibrium computation, yet their success hinges on addressing a critical dilemma: handling enormous decision spaces with limited spatial resources.

Counterfactual Regret Minimization (CFR) serves as the foundational algorithm for computing \(\epsilon\)-Nash equilibrium, but its linear space complexity becomes prohibitive even for moderate-scale games. Consider heads-up limit Texas Hold'em: with \(\approx 10^{13}\) information sets, it demands 523 terabytes of RAM~\cite{johanson2013measuring}—a scale rendering direct CFR implementation physically impossible, let alone larger games like heads-up no-limit. Existing systems thus trade theoretical rigor for practicality via approximations, with the prevailing ``abstraction-solving-translation" paradigm~\cite{gilpin2007potential} as the core approach: compressing the original game into a feasible abstract version, solving for \(\epsilon\)-Nash equilibrium within it, and mapping the solved strategy back to the original game. One key approach within this paradigm, known as information set abstraction, is a pre-training method that groups similar information sets into discrete equivalence classes prior to strategy solving, and uniform strategies are applied to these classes to reduce spatial resource demands. This approach is validated by the aforementioned superhuman AI systems, which leveraged it to defeat top human players.

However, existing information set abstraction algorithms have critical limitations. Relying predominantly on many-to-one mappings, they often introduce arbitrary classification boundaries. As shown in Figure~\ref{fig-motivation}(c), consider two hands in the flop round of Texas Hold'em (see Appendix 
\ref{apdx:heads_up_texas}
for rules)\footnote{All appendix content is available in the extended version.}: hand \(\blacksquare\) (\(\texttt{9}^p_s \texttt{A}^p_s \texttt{9}_h \texttt{J}_s \texttt{Q}_s\))\footnote{h: hearts, d: diamonds, s: spades, c: clubs; superscripts \(^p\) denote private cards (e.g., \(\texttt{8}^p_h \texttt{8}^p_d\)= player's 8s in hearts/diamonds).} and hand $\scalebox{1.9}{\raisebox{-0.15ex}{\text{$\bullet$}}}$ (\(\texttt{T}^p_s \texttt{A}^p_s \texttt{T}_h \texttt{J}_s \texttt{Q}_s\)); both hands forming a pair. While superficially similar, hand $\scalebox{1.9}{\raisebox{-0.15ex}{\text{$\bullet$}}}$ can form a straight flush with \(\texttt{K}_s\), making it inherently stronger than hand \(\blacksquare\) (which can realistically only form a flush with any additional spade card). Current algorithms face a dilemma: keeping these hands--and the information sets they represent--separate wastes the similar strategies that would naturally follow from their inherent similarities, while grouping them overlooks nuanced differences. This binary, low-granularity approach introduces arbitrariness that degrades strategy quality.
\begin{figure}[t]
\centering
\includegraphics[width=0.89\columnwidth]{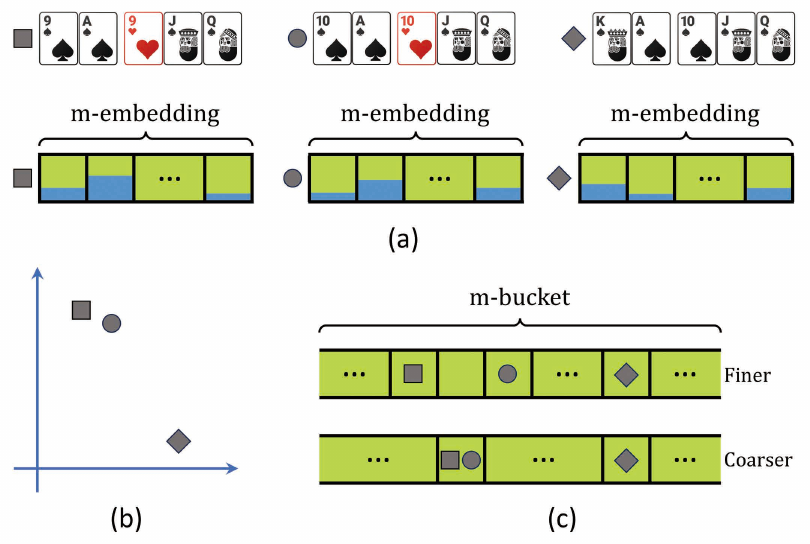} 
\caption{
\textbf{Behavior comparison of hand representations under Embedding CFR and traditional information set abstraction} for hands $\blacksquare$, $\scalebox{1.9}{\raisebox{-0.15ex}{\text{$\bullet$}}}$, and $\scalebox{0.8}{\rotatebox{45}{\text{$\blacksquare$}}}$ in Texas Hold'em. \\
    \textbf{(a)} Embedding CFR maps hands to embedding coordinates, which form an m-dimensional probability distribution where the sum of values across all dimensions equals 1. \\
    \textbf{(b)} Schematic 2D projection of embedding coordinates illustrates the geometric topology between hands, highlighting both similarity (closeness of $\blacksquare$ and $\scalebox{1.9}{\raisebox{-0.15ex}{\text{$\bullet$}}}$) and distinction (separation from $\scalebox{0.8}{\rotatebox{45}{\text{$\blacksquare$}}}$). \\
    \textbf{(c)} Traditional abstraction maps information sets to a fixed number of $m$ abstracted classes, e.g. buckets, forcing binary decisions for these hands: either refining $\blacksquare$ and $\scalebox{1.9}{\raisebox{-0.15ex}{\text{$\bullet$}}}$ into distinct equivalence classes or coarsening them into one. This lack of intermediate states hinders exploitation of inter-information-set similarity for strategy solving.
}
\label{fig-motivation}
\end{figure}

In this paper, we propose \textbf{Embedding CFR}, a novel strategy-solving paradigm for information set abstraction, to address the aforementioned limitations. Unlike traditional methods that map multiple information sets to a single equivalence class, it maps each information set to a multi-dimensional vector—embedding coordinates forming a probability distribution--thereby enhancing abstraction expressiveness. As Figure~\ref{fig-motivation}(a) illustrates, hands \(\blacksquare\), $\scalebox{1.9}{\raisebox{-0.15ex}{\text{$\bullet$}}}$, and $\scalebox{0.8}{\rotatebox{45}{\text{$\blacksquare$}}}$ (the latter a strong straight flush: \(\texttt{K}^p_s \texttt{A}^p_s \texttt{T}_s \texttt{J}_s \texttt{Q}_s\) at the flop) are assigned embedding coordinates representing their confidence in corresponding dimensions. During strategy solving, updates to information sets preferentially influence high-confidence dimensions, ensuring strategies associated with closer coordinates exhibit greater similarity. Figure~\ref{fig-motivation}(b) shows a schematic 2D projection of this embedding, where hand \(\blacksquare\) and hand $\scalebox{1.9}{\raisebox{-0.15ex}{\text{$\bullet$}}}$ cluster closely yet remain subtly distinguishable--yielding correspondingly similar solved strategies--while being distinctly separated from the strong hand $\scalebox{0.8}{\rotatebox{45}{\text{$\blacksquare$}}}$, ensuring minimal mutual influence during solving.

This paper makes the following contributions:
\begin{enumerate}
    \item To our knowledge, at least in poker AI development, ours is the first work to employ a pre-trained embedding approach for information set abstraction.
    \item We present a framework for solving strategies using no-regret optimization under the condition of information set embedding, and provide an approximate analysis of its ability to achieve regret decrease.
    \item We successfully apply Embedding CFR to poker AI development and propose an embedding construction algorithm for poker.
\end{enumerate}

Experiments are conducted in a poker game, comparing with traditional information set abstraction algorithms such as EHS~\cite{gilpin2007better}, PaEmd (the information set abstraction algorithm adopted by Libratus, recognized as a state-of-the-art approach in the field)~\cite{ganzfried2014potential}, and a recent work KrwEmd~\cite{fu2024signal}. The results show that the strategy solved by Embedding CFR has lower exploitability under the same spatial overhead and update iterations, demonstrating the effectiveness of the proposed algorithm.

\subsection{Related Works}

Our work lies within the community of solving imperfect-information extensive-form games via no-regret learning, with CFR~\cite{zinkevich2007regret} as the core strategy-solving algorithm. Extensions follow two distinct directions: sampling-based methods like MCCFR~\cite{lanctot2009monte}, and strategy update-focused optimizations including CFR+~\cite{bowling2015heads}, DCFR~\cite{brown2019solving}, and PCFR+~\cite{farina2021faster}.

For strategy solving under limited space constraints, related efforts include pruning and action abstraction techniques: regret-based pruning \cite{brown2015regret}, best-response pruning \cite{brown2017reduced}, dynamic thresholding pruning \cite{brown2017dynamic}, along with reinforcement learning-based online action abstraction methods like RL-CFR \cite{li2024rl} and EVPA \cite{li2025efficient}.

Closer to our approach, neural network-based regret estimation methods (DeepStack \cite{moravvcik2017deepstack}, ReBeL~\cite{brown2020combining}, Deep CFR~\cite{brown2019deep}, Escher~\cite{mcaleer2023escher}) replace tabular storage of regrets and strategies to simplify solving. They excel in model-free scenarios, where reliance on environmental interactions--which simultaneously drive strategy optimization and information set learning--reduces dependence on complex domain knowledge. In model-known settings, however, this dual reliance creates trade-offs, blunting efficiency compared to pre-training-based methods leveraging prior domain knowledge.

Our closest predecessors are information set abstraction works: lossless abstraction applicable to small-scale games \cite{gilpin2007lossless}, expectation-based EHS \cite{gilpin2007better}, potential-aware methods \cite{gilpin2007potential, ganzfried2014potential} (notably PaEmd), and the history trajectory-based KrwEmd \cite{fu2024signal}. We share with these predecessors a pre-training-then-solving paradigm--analogous to NLP pipelines where word embedding or tokenization precedes semantic processing--by first preprocessing information sets before strategy solving.

\section{Background and Notation}
\label{sec:background}

An imperfect-information extensive-form game (IIEFG) is defined by the tuple $(\mathcal{N}, H, A, \mathcal{P}, u, \sigma_c, \mathcal{I})$. The set $\mathcal{N} = \{1, \dots, N\} \cup \{c\}$ represents the finite set of players; in this work, we focus on the two-player case where $N = 2$, with $c$ being a special \textbf{chance} player whose actions model stochastic events. The set $H$ consists of histories (also referred to as nodes), where each $h \in H$ represents a sequence of actions from the game's start, with the empty sequence $h^0$ denoting the unique initial history. We write $h \sqsubseteq h'$ if $h$ is a prefix of $h'$, and $h \sqsubset h'$ if it is a strict prefix; $h \cdot a$ denotes the history formed by appending action $a \in A(h)$ to $h$. Here, $A(h)$ is the set of available actions for non-terminal histories $h\in H\setminus Z$, with $Z \subset H$ being the set of terminal histories that end the game. The player function $\mathcal{P}: H \setminus Z \to \mathcal{N}$ assigns a unique acting player to each non-terminal history. The utility function $u = (u_1, \dots, u_N)$ gives each player $i \in \mathcal{N}$ a real-valued payoff $u_i(z)$ for every terminal history $z \in Z$. In the two-player zero-sum setting considered here, $u_1(z) = -u_2(z)$ for all $z \in Z$. The function $\sigma_c$ models the chance player's behavior by specifying, for each history $h$ where $\mathcal{P}(h) = c$, a probability distribution $\sigma_c(h, a)$ over actions $a \in A(h)$, capturing the game's stochastic nature. 

A central concept in imperfect-information extensive-form games is the \textbf{information set (infoset)}, which captures the inherent uncertainty a player faces regarding the exact game history when making a decision. Formally, $\mathcal{I} = \bigcup_{i \in \mathcal{N} \setminus \{c\}} \mathcal{I}_i$ represents the collection of infosets, where each $\mathcal{I}_i$ is a partition of the set $H_i = \{h \in H \setminus Z \mid \mathcal{P}(h) = i\}.$ That is, each history in which player $i$ takes an action belongs to exactly one infoset. Within any $I \in \mathcal{I}_i$, the player cannot distinguish between the histories $h, h' \in I$, implying the following consistency conditions:
\[
\forall h, h' \in I: \begin{cases}\mathcal{P}(h) = \mathcal{P}(h') = i=\mathcal{P}(I),\\ A(h) = A(h')=A(I).\end{cases}
\]
For infoset $I$, we define $Z_I$ as the set of terminal histories that can be reached from some $h \in I$. Throughout this work, we assume the game adheres to the property of \textbf{perfect recall}, which stipulates that players retain complete knowledge of all information they have previously encountered. More formally, for player~$i$, if two different histories are not part of the same infoset, then no subsequent continuations of these histories can be grouped within the same infoset for that player. An important implication of this property is that each branch from the root to the leaves of the game tree traverses any given infoset at most once. We use the notation \(z[I]\) to denote the history prefix of a terminal history $z$ that passes through infoset $I$. This notation is undefined (or invalid) if the path from \(h^0\) to $z$ does not traverse $I$. 

A player's strategy $\sigma_i \in \Sigma_i$ assigns a probability $\sigma_i(I, a)$ to each action $a$ at every infoset $I \in \mathcal{I}_i$. A strategy profile $\sigma = \langle\sigma_1, \dots, \sigma_N\rangle$ specifies a complete strategy for each player. Given $\sigma$, the reach probability of a history $h \in H$ is defined as $\pi^\sigma(h) = \pi^\sigma(h^0, h)$, where $\pi^\sigma(h, h') = \prod_{i \in \mathcal{N}} \pi_i^\sigma(h, h')$, with 
\[
\pi_i^\sigma(h, h') = 
\begin{cases}
0, &\!\!\!\text{if } h \not\sqsubseteq h' \\
\prod\limits_{h'':h \sqsubseteq h'', h'' \cdot a \sqsubseteq h', \mathcal{P}(h'') = i} \!\!\!\sigma_i(h'', a), & \!\!\!\text{if } h \sqsubseteq h'
\end{cases}.
\]
The contribution of player~$j$ to the reach probability of history~$h$ is denoted by $\pi_i^\sigma(h) = \pi_i^\sigma(h^0, h)$. In later sections, we may slightly abuse notation by writing $\sigma_i(h, a)$ to denote the probability assigned to action~$a$ at the infoset containing~$h$. We similarly use the subscript $-i$ in both $\pi_{-i}^\sigma$ and $\sigma_{-i}$ to refer to the reach probability and strategy profile, respectively, of all players except player~$i$.

In game theory, the \textbf{Nash equilibrium} is a crucial solution concept: a state where no player can enhance their payoff by unilaterally altering their strategy, and to some extent the optimal solution. The payoff for player $i$ under strategy profile $\sigma$ is $u_i(\sigma) = \sum_{z \in Z} u_i(z) \pi^\sigma(z)$. The \textbf{best response} value for player $i$ against $\sigma_{-i}$ is $b_i(\sigma_{-i}) = \max_{\sigma_i' \in \Sigma_i} u_i(\sigma_i', \sigma_{-i})$, the max payoff for optimal response. A strategy profile $\sigma^*$ is an \textbf{$\epsilon$-Nash equilibrium} if $\forall i\in \mathcal{N}\backslash \{c\}$, $u_i(\sigma^*) + \epsilon \geq b_i(\sigma^*_{-i})$. When $\epsilon = 0$, it's a standard Nash equilibrium. In two-player zero-sum games, the \textbf{exploitability} of $\sigma$ is $\epsilon^\sigma = b_1(\sigma_2) + b_2(\sigma_1)$, measuring players' vulnerability to an optimal opponent. Lower exploitability means closer to equilibrium.

\subsection{Counterfactual Regret Minimization}
\label{subsec:cfr}

Counterfactual Regret Minimization (CFR), also referred to as Vanilla CFR, is an iterative algorithm for computing approximate Nash equilibrium in IIEFGs by minimizing \textbf{counterfactual regret}~\cite{zinkevich2007regret}. For player \(i\), the total regret after \(T\) iterations is defined as
\[
R^T_i = \frac{1}{T}\max_{\sigma_i'\in \Sigma_i}\sum_{t=1}^T\left(u_i(\sigma_i',\sigma^t_{-i}) - u_i(\sigma^t)\right).
\]
In two-player zero-sum IIEFGs, if both players satisfy \(R^T_i \leq \epsilon\), then the average strategy forms a \(2\epsilon\)-Nash equilibrium~\cite{waugh2009abstraction}.

CFR avoids the difficulty of directly minimizing global regret by instead minimizing cumulative counterfactual regret at each infoset. It defines the \textbf{counterfactual value}:
\begin{equation*}
v_i^\sigma(I) = \sum_{z \in Z_I} \pi_{-i}^\sigma(z[I]) \pi^\sigma(z[I], z) u_i(z),
\end{equation*}
and the \textbf{immediate counterfactual regret}:
\begin{equation}
r^T(I,a) = v_{\mathcal{P}(I)}^{\sigma^T_{|I\to a}}(I) - v_{\mathcal{P}(I)}^{\sigma^T}(I),
\label{eq:immediate_regret}
\end{equation}
where \(\sigma_{|I\to a}\) is a strategy profile identical to \(\sigma\) except that player \(\mathcal{P}(I)\) always plays action \(a\) at \(I\).

Then, the \textbf{cumulative counterfactual regret} 
\begin{equation}
R^T(I,a) = \frac{1}{T}\sum_{t=1}^T r^t(I,a)    \label{eq:cumulative_regret}
\end{equation} is used to update the immediate strategy via \textbf{regret matching}~\cite{hart2000simple}:
\begin{align*}
\sigma^{T+1}_{\mathcal{P}(I)}(I, a) = 
\begin{cases}
\frac{R^T(I,a)_+}{~~\sum\limits_{\mathclap{a'\in A(I)}} \left(R^T(I,a')_+\right)}, & \text{if } \sum\limits_{\mathclap{a'\in A(I)}} \left(R^T(I,a')_+\right) > 0, \\
\frac{1}{|A(I)|}, & \text{otherwise,}
\end{cases}
\end{align*}
where $x_+= \max(x,0)$.

This process causes \(\sum_{I\in \mathcal{I}_i}(\max_{a\in A(I)} R^T(I,a)_+)\) to converge to zero at a rate of \(\mathcal{O}(1/\sqrt{T})\), which is an upper bound of \(R_i^T\). Therefore, the average strategy profile \(\bar{\sigma}^T = \langle \bar{\sigma}_1^T, \bar{\sigma}_2^T \rangle\) converges to an \(\epsilon\)-Nash equilibrium, where
\begin{equation*}
\bar{\sigma}_{i=\mathcal{P}(I)}^T(I, a) = \frac{\sum_{t = 1}^T \pi_i^{\sigma^t}(I)\sigma_i^t(I, a)}{\sum_{t = 1}^T \pi_i^{\sigma^t}(I)}. 
\end{equation*}

\section{Embedding CFR}

This section introduces the Embedding CFR algorithm, a strategy-solving approach that leverages embedding for information set abstraction. We first extend the definition of information set abstraction (Section \ref{sec:infoset_abstraction_reformulation}), then detail the algorithm's driving process (Section \ref{sec:driving_process}), and finally analyze its regret convergence trends in low-dimensional spaces under restricted conditions (Section \ref{sec:convergence_analysis}).

\subsection{Problem Modeling} \label{sec:infoset_abstraction_reformulation}

We begin our approach by redefining information set abstraction to expand its conceptual scope. Specifically, the abstraction takes \textbf{info-blocks} as its primary objects of operation. In an IIEFG, an \textbf{info-block partition} is defined as a partition of a player $i$’s infosets collection $\mathcal{I}_i$ into non-overlapping groups $\mathcal{J} = \{J_1, J_2, \dots, J_K\}$. Each info-block $J_k$ aggregates infosets that are strategically relevant or similar, under the constraint that all infosets $I, I' \in J_k$ share the same action space ($A(I) = A(I')$). The construction of info-block partitions is non-unique and must be predefined manually prior to information set abstraction, requiring game-specific analysis to determine valid groupings; we will discuss the construction of info-block partitions in the context of specific games in Section \ref{sec:partition_infoblock_chance}.

\begin{definition}[Extended Information Set Abstraction]
For an info-block \( J = \{I_1, \dots, I_n\} \) containing \( n \) infosets of player \( i \), \textbf{extended information set abstraction} is defined as the strategic association of $J$ with a set of \( m \) advisors \( E = \{e_1, e_2, \dots, e_m\} \). Each advisor \( e_p \in E \) is assigned a strategy function \( \sigma(e_p, \cdot) \colon A(J) \to [0,1] \) forming a valid probability distribution over \(A(J)\), where \( A(J) \) denotes the shared action space of all infosets in \( J \). The strategy of player \( i \) at an infoset \( I_q \in J \) is determined by an \( m \)-ary aggregation function \( f_q \), such that for any action \( a \in A(J) \):  
\[
\sigma(I_q, a) = f_q\Big( \sigma(e_1, a), \sigma(e_2, a), \dots, \sigma(e_m, a) \Big).
\]
\end{definition}

Advisors act as fixed-strategy decision-makers, adopting a uniform strategy across all infosets in $J$; for each infoset, the player predefines a method to construct its strategy by aggregating these advisor strategies. Extended information set abstraction leverages this structure to reduce storage overhead: instead of maintaining $n$ distinct strategies for each infoset in $J$, it retains $m$ advisor strategies over $E$. When \(m < n\), this yields a space complexity of \(\mathcal{O}(m \cdot |A(J)|)\), outperforming the \(\mathcal{O}(n \cdot |A(J)|)\) requirement of traditional abstractions. Notably, this framework generalizes traditional information set abstraction: if, for each \(I_q \in J\), the aggregation function \(f_q\) reduces to an identity mapping onto a single advisor \(e_p\) (i.e., \(\sigma(I_q, a) = \sigma(e_p, a)\) for all \(a \in A(J)\)), the two abstractions coincide.

\subsection{Driving Process for Embedding CFR}
\label{sec:driving_process}

For ease of discussion, while \(J = \{I_1, \dots, I_n\}\) and \(E = \{e_1, e_2, \dots, e_m\}\) are originally set-theoretic concepts for player $i$, we arrange them into column vectors \(\mathbf{J} = [I_1, \dots, I_n]^\top\) and \(\mathbf{E} = [e_1, \dots, e_m]^\top\) to facilitate the application of linear algebra tools. Furthermore, for any function $g$, its action on \(\mathbf{J}\) or \(\mathbf{E}\) is defined element-wise as

\begin{equation}
\left\{
\begin{aligned}
g(\mathbf{J}) &= [g(I_1), g(I_2),\dots, g(I_n)]^\top, \\
g(\mathbf{E}) &= [g(e_1), g(e_2),\dots, g(e_m)]^\top.
\end{aligned}
\right.
\label{eq:element_wise}
\end{equation}

Building on this notation, we propose \textbf{Embedding CFR}, which models the extended information set abstraction problem as a strategy voting framework, where a \textbf{embedding matrix} \(\mathbf{\Phi} = [\phi]_{m\times n}\) is constructed. Each non-negative element \(\phi_{p,q} \geq 0\)--with the constraint that \(\sum_{p=1}^m \phi_{p,q} = 1\) for each $q$--represents the confidence in trusting advisor \(e_p\) at infoset \(I_q\); here, \(\mathbf{\Phi}_{\cdot,q}\) denotes the embedding coordinates of infoset \(I_q\). Given a strategy \(\sigma_i^T(\mathbf{E})\) in the advisor space (also called embedding space), decisions in the original strategy space are made via the mapping 
\begin{equation}
\sigma_i^T(\mathbf{J}, a) = \mathbf{\Phi}^\top \sigma_i^T(\mathbf{E}, a) \label{eq:immediate_strategy_embedding}
\end{equation} 
for each \(a\in A(J)\).

To solve games under this abstraction, Embedding CFR operates as a no-regret learning algorithm that draws insights from Vanilla CFR \cite{zinkevich2007regret}: as illustrated in the right half of Figure \ref{fig:embedding_cfr_driving} and as we have detailed in Section \ref{subsec:cfr}, Vanilla CFR updates the cumulative counterfactual regret \(R^T(I)\) of the target infoset $I$ at iteration $T$, then derives the immediate strategy \(\sigma_i^{T+1}(I)\) for next-iteration online interaction with the environment via regret matching, after which the regret for the new iteration is updated—forming a driving process. The strategy to be evaluated, which exhibits favorable convergence properties, is the average strategy \(\bar{\sigma}_i^T(I)\). Embedding CFR extends this process by addressing how to drive it within the newly introduced advisor space.

\begin{figure}[t]
\centering
\includegraphics[width=\columnwidth]{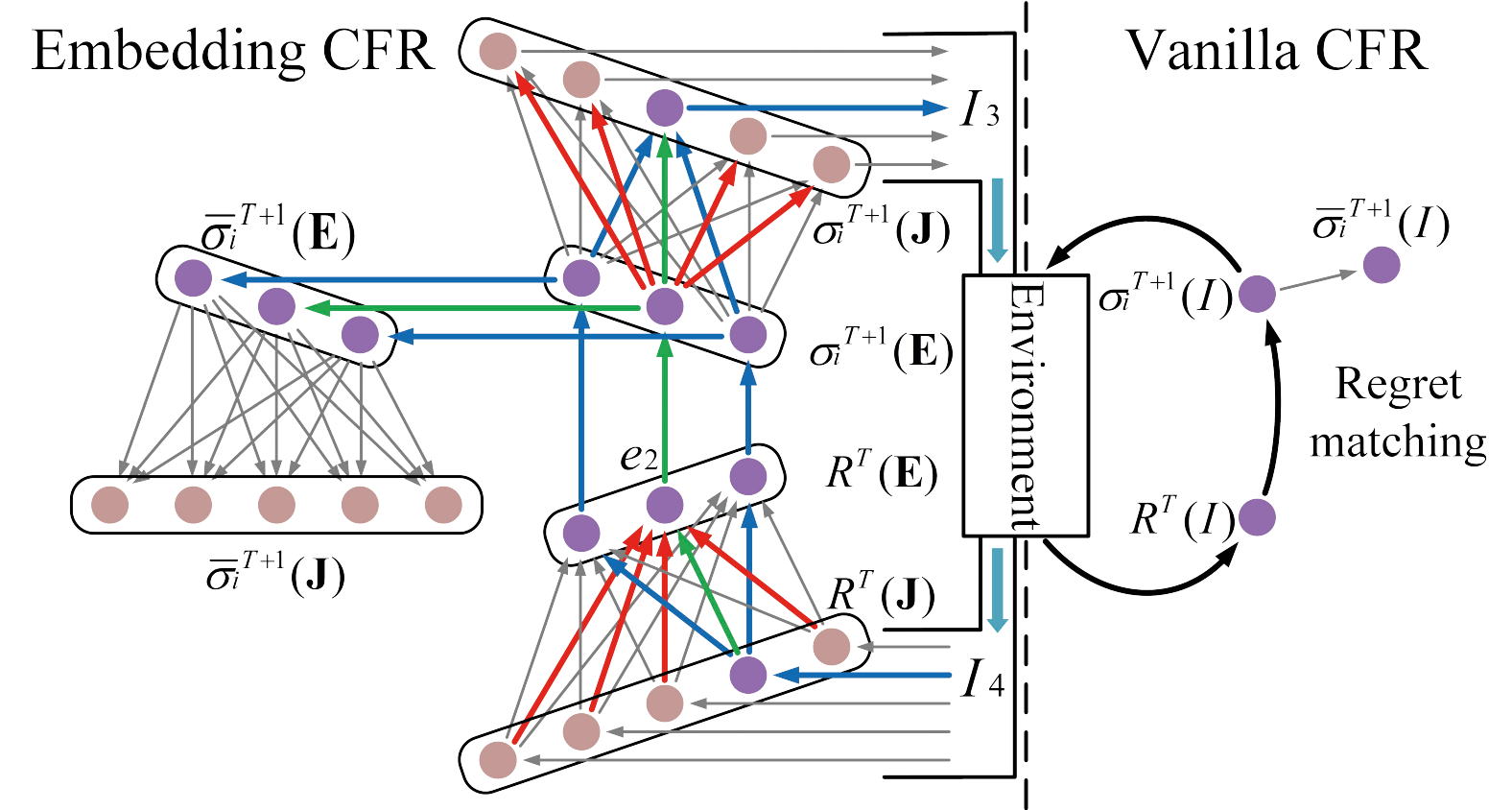} 
\caption{Schematic comparison of driving processes: Embedding CFR vs. Vanilla CFR}
\label{fig:embedding_cfr_driving}
\end{figure}

As illustrated in the left half of Figure \ref{fig:embedding_cfr_driving}, when collecting regrets from interaction with the environment, we first decompose regrets from the original space into the advisor space based on embedding coordinates. Using Equations \eqref{eq:immediate_regret}, \eqref{eq:cumulative_regret} and \eqref{eq:element_wise}, we define the embedded immediate regret as \(r^{T}(\mathbf{E},a) = \mathbf{\Phi} \cdot r^{T}(\mathbf{J}, a)\) , and the embedded cumulative regret as \(R^{T}(\mathbf{E},a) = \frac{1}{T} \sum_{t=1}^T r^{t}(\mathbf{E},a)\). From these, it follows that the embedded cumulative regret satisfies:
\begin{equation}
R^{T}(\mathbf{E},a) = \mathbf{\Phi} R^{T}(\mathbf{J}, a). \label{eq:embed_cumulative_regret}
\end{equation}

Each advisor \(p = 1, \dots, m\) then iteratively updates its embedded immediate strategy via regret matching based on its embedded cumulative regret:

\begin{align}
&\sigma^{T+1}_{i}(e_p, a) = \nonumber \\
&\begin{cases}
\frac{R^T(e_p,a)_+}{~~\sum\limits_{\mathclap{a'\in A(J)}} \left(R^T(e_p,a')_+\right)}, & \text{if } \sum\limits_{\mathclap{a'\in A(J)}} \left(R^T(e_p,a')_+\right) > 0, \\
\frac{1}{|A(J)|}, & \text{otherwise,}
\end{cases} \label{eq:regret_matching_embedding}
\end{align}

The embedded immediate strategy is, on one hand, mapped back to the immediate strategy in the original space via Equation \eqref{eq:immediate_strategy_embedding} and used to interact with the environment to initiate a new iteration; on the other hand, it accumulates into the embedded average strategy in the advisor space:
\begin{equation}
\bar{\sigma}^{T+1}(\mathbf{E}, a) = \frac{T}{T+1}\bar{\sigma}^{T}(\mathbf{E}, a) + \frac{1}{T+1}\sigma^{T+1}(\mathbf{E}, a)
\end{equation}

The average strategy in the original space is recovered via the formula 
\begin{equation}
    \bar{\sigma}_i^T(\mathbf{J}, a) = \mathbf{\Phi}^\top \bar{\sigma}_i^T(\mathbf{E}, a), \label{eq:average_strategy_embedding}
\end{equation}
which serves as the final learned strategy after the iterative process.

It is important to note that while regret embedding (Equation \eqref{eq:embed_cumulative_regret}) and strategy recovery (Equations \eqref{eq:immediate_strategy_embedding}, \eqref{eq:average_strategy_embedding}) involve full matrix multiplications--specifically of sizes \(m \times n\) with \(n \times m\) and \(n \times m\) with \(m \times n\)--this might initially suggest unavoidable \(\mathcal{O}(n)\) or higher space complexity. Our solution addresses this by framing regret embedding as a sampling process and strategy recovery as a query procedure. Rather than operating on the full set of infosets, we dynamically engage a subset of \(l < n\) infosets \(I_{(1)}, \dots, I_{(l)}\) (mapping to indices \(j_1, \dots, j_l\) in \(\mathbf{J}\)) at each iteration. As visualized in Figure \ref{fig:embedding_cfr_driving}, this selective focus is concretely illustrated: iteration $T$ activates only the single infoset \(I_4\), while iteration \(T+1\) shifts to \(I_3\).

This design yields a critical efficiency gain: throughout the process, storage is allocated solely to advisor-space quantities \(R^T(\mathbf{E}, a)\), \(\sigma_i^{T}(\mathbf{E},a)\), and \(\bar{\sigma}_i^T(\mathbf{E}, a)\) for each \(a \in A(J)\), with no need to reserve space for their original-space counterparts. To operationalize the sampling-based embedding, we use the \(\tilde{R}^{T}(\mathbf{E}, a)\) in place of \(R^{T}(\mathbf{E}, a)\):
\begin{align}
\tilde{R}^{T}(\mathbf{E}, a) = \frac{T-1}{T} \tilde{R}^{T-1}(\mathbf{E}, a) + \frac{1}{T} \sum_{k=1}^l \mathbf{\Phi}_{\cdot, j_k}r^{T}(I_k, a). \nonumber 
\end{align}
Correspondingly, the query-based recovery retrieves original-space strategies from the advisor space via:\[\sigma_i^{T}(I_{(k)}, a) =  (\mathbf{\Phi}_{\cdot, j_k})^\top \sigma_i^T(\mathbf{E}, a).
\]
An additional consideration is how to efficiently retrieve the embedding coordinates  \(\mathbf{\Phi}_{\cdot, j_k}\) for each infoset \(I_{(k)}\)--this will be discussed in the next section. Notably, each \(\mathbf{\Phi}_{\cdot, j_k}\) only incurs \(\mathcal{O}(m)\) space overhead, which does not significantly add to overall storage requirements. The blue and green arrows in Figure \ref{fig:embedding_cfr_driving} illustrate these dynamic pathways, with purple nodes representing the quantities that need to be stored in this iteration—emphasizing that this paired sampling-and-query mechanism ensures Embedding CFR operates with \(\mathcal{O}(m)\) space complexity. Pseudocode for poker-specific Embedding CFR is in Appendix 
\ref{apdx:pseudocode}
\negmedspace.

\subsection{Convergence Analysis} \label{sec:convergence_analysis}
Analyzing the convergence of the Embedding CFR algorithm presents significant challenges, which we acknowledge upfront. 

First, establishing the convergence of cumulative regrets in the original space remains non-trivial: embedding inherently introduces information loss, and the quality of solutions recovered in the original space is tightly coupled with the construction of the embedding matrix \(\mathbf{\Phi}\). To illustrate this complexity, consider an extreme scenario where all infosets rely on a single advisor $p$ with full confidence. In such cases, the iterative updates would lack the diversity needed to explore the solution space effectively, making it implausible to converge to a high-quality solution.

Second, even in the advisor space, proving that the embedded cumulative regret of a single advisor stably converges to zero remains challenging. This is because each infoset is simultaneously shaped by all advisors--a dynamic where their update processes are mutually interdependent, with adjustments to one advisor rippling through and altering the convergence trajectory of others.


Fortunately, we can show that when an advisor acts in isolation, Embedding CFR ensures its regret decreases. Embedding CFR can be conceptualized as follows: in iteration $T$, when player $i$ encounters an infoset \(I_q \in J\) ($J$ is an info-block relevant to player $i$) during the game, a random selection is made to decide with probability $\phi_{p,q}$ to act according to the strategy \(\sigma^T_i(e_p)\). Consider an approximate sampling scenario in Embedding CFR, each iteration strictly needs to consider all infosets in $J$; and the key difference is that for each infoset, a selection must be made among \(e_1, \dots, e_m\)--choosing to interact and update according to a specific advisor's strategy. At a particular iteration $T$, it happens that player $i$ selects \(e_p\) for every infoset \(I_q\) (as in Figure \ref{fig:embedding_cfr_driving}’s green/red arrows, where all infosets choose \(e_2\)); this approximates an advisor’s isolated impact on its regret.

\begin{proposition} \label{thm:regret_decrease}
Let \(S^T_p = \sum_{a\in A(J)} \left(R^T(e_p, a)_+\right)^2\) and \(\Delta_J = \max_{\substack{\sigma\in \Sigma \\ I\in J\\ a,a'\in A(J)}}
\bigl|v_{\mathcal{P}(I)}^{\sigma_{|I\to a}}(I) - v_{\mathcal{P}(I)}^{\sigma_{|I\to a'}}(I)\bigr|\). In the aforementioned scenario, the following holds: If \(S^T_p \leq \frac{n|A(J)| \cdot \Delta_J^2 \cdot \|\mathbf{\Phi}_{p,\cdot}\|_2^2}{T}\), then \(S^{T+1}_p \leq \frac{n|A(J)| \cdot \Delta_J^2 \cdot \|\mathbf{\Phi}_{p,\cdot}\|_2^2}{T+1}\). Otherwise, \(S^{T+1}_p \leq \frac{T}{T+1}S^{T}_p\) (see Appendix 
\ref{apdx:proofs} 
for the proof).
\end{proposition}

In Proposition \ref{thm:regret_decrease}, \(S^T_p\) quantifies the overall embedded cumulative positive regret level associated with \(e_p\). The proposition shows that when \(S^T_p\) is less than a specified threshold, \(S^{T+1}_p\) can be bounded by a definite decreasing value. When \(S^T_p\) exceeds this threshold, although no explicit upper bound for \(S^{T+1}_p\) can be provided, it is guaranteed to be a decreasing quantity relative to \(S^T_p\). This conclusion offers an approximate analysis of regret reduction when a single advisor acts independently. By extension, it provides intuition that even when all advisors act together (with mutual influences), each advisor's regret tends to exhibit a decreasing trend, meaning Embedding CFR can achieve regret-decreasing learning within the advisor space.

\section{Application To Poker Games}

We now discuss applying the Embedding CFR algorithm to strategy solving in typical imperfect-information extensive-form games like poker, using domain knowledge.

\subsection{Partitioning of Info-blocks by Chance Actions} \label{sec:partition_infoblock_chance}

\begin{figure}[H]
  \centering
  \begin{subfigure}[b]{0.23\textwidth}
    \centering
    \includegraphics[width=0.98\textwidth]{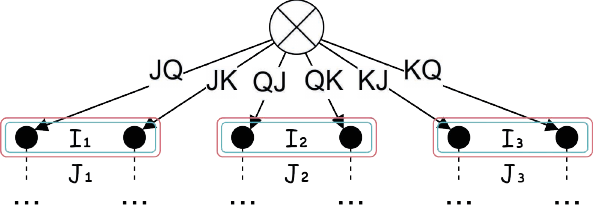}
    \caption{\(I_1, I_2, I_3\) are partitioned into distinct info-blocks \(J_1\) and \(J_2\).}
  \end{subfigure}
   \hfill
  \begin{subfigure}[b]{0.23\textwidth}
    \centering
    \includegraphics[width=0.98\textwidth]{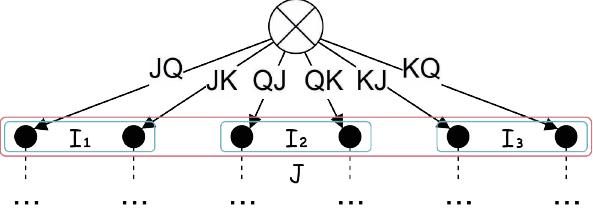}
    \caption{\(I_1, I_2, I_3\) are partitioned into the same info-block $J$.}
  \end{subfigure}
  \caption{Two info-block partitioning schemes for player 1's infosets \(I_1, I_2, I_3\) in Kuhn Poker, noting that all sets satisfy action-space consistency (\(A(I_1) = A(I_2) = A(I_3)\)).}
  \label{fig:kuhn_info_blocks_partition}
\end{figure}

First, we address the unresolved issue of constructing info-block partitions from Section \ref{sec:infoset_abstraction_reformulation}.

On one hand, as noted, such partitions (exemplified in Kuhn Poker; Figure \ref{fig:kuhn_info_blocks_partition}, rules in Appendix 
\ref{apdx:kuhn}
\negmedspace) admit non-unique constructions, and not all partitions facilitate strategy solving. Trivially isolating each infoset into individual blocks (Figure \ref{fig:kuhn_info_blocks_partition}(a)) undermines the generalization benefits of correlations among similar infosets; intuitively, the framework functions better with more aggregated infosets per block. Conversely, merging all infosets with matching action spaces is flawed: in heads-up limit Texas Hold'em (rules in Appendix 
\ref{apdx:heads_up_texas}
\negmedspace), pre-flop and post-flop infosets may share action spaces (call/raise/fold) but are strategically distinct, making merger inappropriate.

On the other hand, algorithmically, Embedding CFR operates per info-block, requiring an embedding matrix \(\mathbf{\Phi}\) for each--a labor-intensive process. If partitions allowed matrices to be constructed for a subset of info-blocks, with others extendable from these, workload would drop markedly.

Poker games possess a key property that facilitates info-block partitioning while addressing both challenges mentioned above: history indistinguishability within an infoset stems from chance actions (e.g., dealing), while non-chance actions (e.g., betting) are public and preserve strategic distinctness. Critically, all histories in the same infoset share an identical non-chance action trace (a result of \textbf{perfect recall} in poker). This formalizes our partitioning rule: info-blocks in poker games are defined as groups of infosets whose contained histories share the same non-chance action trace.

Consider a concrete example from HULHE, where a game history $h$ can be decomposed into interleaved chance (card dealing) and non-chance (betting) segments:
\[
h = \underbrace{\texttt{8}^1_h\texttt{8}^1_d\texttt{J}^2_s\texttt{Q}^2_s}_{\mathclap{\text{Chance: Private Cards\quad}}} \cdot 
   \overbrace{\texttt{rRc}}^{\mathclap{\text{Non-chance Actions}}} \cdot
   \underbrace{\texttt{2}_s\texttt{J}_s\texttt{8}_h}_{\mathclap{\text{\quad Chance: Community Cards}}} \cdot
   \overbrace{\texttt{Cr}}^{\mathclap{\text{Non-chance Actions}}} \in I \in \mathcal{I}_2, \footnote{\texttt{c/r/f} (Player 1, small blind) and \texttt{C/R/F} (Player 2, big blind) denote check, raise, and fold, respectively.}
\]
For two additional histories \(h'\) and \(h''\):
\begin{itemize}
    \item $h' = \texttt{7}^1_h\texttt{7}^1_d\texttt{J}^2_s\texttt{Q}^2_s \cdot \texttt{rRc}\cdot \texttt{2}_s\texttt{J}_s\texttt{8}_h\cdot \texttt{Cr} \in I$,
    \item $h'' = \texttt{7}^1_h\texttt{5}^1_d\texttt{8}^2_s\texttt{4}^2_s \cdot \texttt{rRc}\cdot \texttt{J}_s\texttt{Q}_s\texttt{3}_h\cdot \texttt{Cr} \in I' \ne I$,
\end{itemize}
all three histories (and their containing infosets \(I, I'\)) belong to the same info-block $J$ by our partitioning rule, as they share the identical non-chance trace: \(\mathcal{H}_{-c}(h) = \mathcal{H}_{-c}(h') = \mathcal{H}_{-c}(h'') = \emptyset \cdot \texttt{rRc}\cdot \emptyset \cdot \texttt{Cr},\)
where \(\mathcal{H}_{-c}(\cdot)\) denotes the operator that retains only non-chance actions (replacing chance segments with \(\emptyset\)).

Furthermore, this partitioning rule directly addresses the challenge of constructing embedding matrices for each info-block. For example, consider histories \(\dot{h}, \dot{h}', \dot{h}''\) with the same chance traces as \(h, h', h''\) but a different non-chance trace (\(\emptyset\cdot \texttt{rC}\cdot\emptyset\cdot \texttt{C}\))--the infosets to which these histories belong come from another info-block \(\dot{J}\):
\begin{itemize}
    \item $\dot{h} = \texttt{8}^1_h\texttt{8}^1_d\texttt{J}^2_s\texttt{Q}^2_s \cdot \texttt{rC}\cdot \texttt{2}_s\texttt{J}_s\texttt{8}_h\cdot \texttt{C} \in \dot{I} \in \dot{J}$,
    \item $\dot{h}' = \texttt{7}^1_h\texttt{7}^1_d\texttt{J}^2_s\texttt{Q}^2_s \cdot \texttt{rC}\cdot \texttt{2}_s\texttt{J}_s\texttt{8}_h\cdot \texttt{C} \in \dot{I} \in \dot{J}$,
    \item $\dot{h}'' = \texttt{7}^1_h\texttt{5}^1_d\texttt{8}^2_s\texttt{4}^2_s \cdot \texttt{rC}\cdot \texttt{J}_s\texttt{Q}_s\texttt{3}_h\cdot \texttt{C} \in \dot{I}' \in \dot{J}$.
\end{itemize}

Notably, $J$ and \(\dot{J}\) are in one-to-one correspondence: their infosets (e.g., \(I \leftrightarrow \dot{I}\), \(I' \leftrightarrow \dot{I}'\)) are counterparts linked by identical chance traces. Since history indistinguishability within poker infosets stems solely from such chance actions (not non-chance actions), we construct embedding matrix \(\mathbf{\Phi}\) based on chance actions (i.e., card dealing)--details in Section \ref{sec:embedding_matrices}. Given that $J$ and \(\dot{J}\) share this chance-based structure via their corresponding infosets, \(\mathbf{\Phi}\) built for $J$ can be directly reused for \(\dot{J}\), substantially reducing construction workload.

\subsection{Embedding Matrices: Construction and Deployment in Poker Games} \label{sec:embedding_matrices}

We now discuss the construction of embedding matrices in poker games and their deployment in Embedding CFR iterations. As established in Section \ref{sec:partition_infoblock_chance}, these matrices are grounded in chance actions—specifically, players’ hands in poker games.

Poker hands exhibit a comparative structure (not strictly a total order, as not all hands are directly comparable). In the final betting round, hands form a total order enabling definitive comparisons; in earlier rounds, their relative strength is inferred from outcome distributions when rolled out to the final round (simulating remaining card deals yields probabilistic performance measures). This comparative strength captures both distinctions and connections between infosets: hands with similar strength profiles share strong strategic ties, while divergent profiles indicate distinctiveness—forming the basis of our embedding  matrix construction.

However, storing a complete \(\mathbf{\Phi}\) would require \(m \times n\) space, which is infeasible given poker’s massive hand scale (\(n \approx 5 \times 10^{10}\) in Texas Hold'em). Thus, instead of constructing \(\mathbf{\Phi}\) in full, we aim to generate the embedding coordinates \(\mathbf{\Phi}_{\cdot, q}\) on-the-fly for each infoset \(I_q\).

\begin{figure}[t]
\centering
\includegraphics[width=0.8\columnwidth]{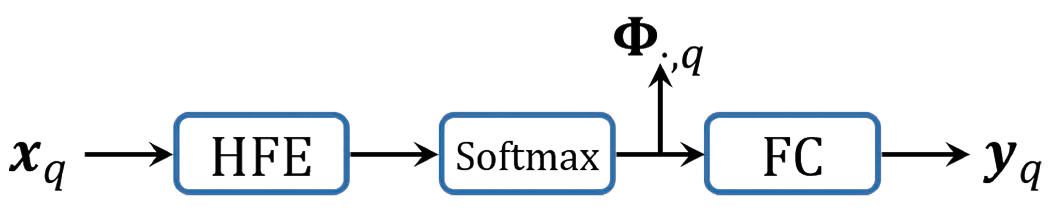} 
\caption{The network architecture of HandEbdNet.}
\label{fig:hand_embedding_net}
\end{figure}

Neural networks are well-suited for the aforementioned supervised learning and dynamic embedding generation scenarios. To this end, we propose HandEbdNet, an embedding network illustrated in Figure \ref{fig:hand_embedding_net}. It takes as input a hand tensor \(\mathbf{x}_q\)—a structured encoding of a hand’s suits, ranks, and round-wise presence (corresponding to \(I_q \in J\))—and outputs a strength tensor \(\mathbf{y}_q\) capturing comparative strength (e.g., win-rate distributions across rounds).

Functionally, HandEbdNet comprises three core components: (1) a hand feature encoder (HFE)--a composite module of multiple networks--that maps \(\mathbf{x}_q\) to an m-dimensional intermediate feature vector; (2) a softmax layer transforming this vector into an m-dimensional probability distribution, viewed as \(\mathbf{\Phi}_{\cdot, q}\); and (3) a final fully connected layer (FC) that transforms this distribution, with supervision from \(\mathbf{y}_q\) forming the learning objective.

The core insight is using \(\mathbf{y}_q\) to induce discriminative \(\mathbf{\Phi}_{\cdot, q}\): hands with similar strength characteristics (reflected in \(\mathbf{y}_q\)) yield proximate \(\mathbf{\Phi}_{\cdot, q}\), capturing meaningful structural relationships between infosets in the advisor space. In poker scenarios, traversing all possible hands (rather than storing them) is feasible, and no entirely new hands will be encountered during embedding coordinate generation. Thus, HandEbdNet is not prone to overfitting in both training and deployment. Details on constructing \(\mathbf{x}_q\), \(\mathbf{y}_q\), and HandEbdNet’s full architecture are provided in Appendix 
\ref{apdx:handebdnet_implementation}
\negmedspace.

\section{Experiment}

We conducted experiments in the Numeral211 Hold'em environment, a simplified variant of Texas Hold'em poker proposed by \citet{fu2024signal}. This game uses a 40-card deck and includes 3 betting rounds, retaining sufficient complexity—particularly in hand diversity—while reducing the overall game scale. This makes it an ideal testbed for researching information set abstraction. Detailed rules are provided in Appendix 
\ref{apdx:numeral211}
\negmedspace.

We primarily compare strategies generated by (extended) information set abstraction methods—including the classical EHS, PaEmd (an algorithm successfully applied in DeepStack and Libratus), and the novel KrwEmd proposed by \citet{fu2024signal}—when paired with Vanilla CFR for solving, alongside our Embedding CFR, which simultaneously handles information set abstraction and strategy solving.

To ensure a fair comparison, all algorithms are maintained at the logical spatial resources (in terms of the number of abstracted information sets or the size of the advisor space). For the Numeral211 Hold'em environment, the number of hand combinations across its three betting rounds is \(\binom{40}{2}\), \(\binom{40}{2} \times \binom{38}{1}\), and \(\binom{40}{2} \times \binom{38}{1} \times \binom{37}{1}\) respectively. The number of structurally equivalent hand strength classes (via lossless isomorphism~\cite{gilpin2007lossless}) across the three rounds is 100, 2250, and 3957, as determined by \citet{waugh2013fast}. Regarding the logical spatial resources used by abstraction algorithms and Embedding CFR, they are restricted to 225 and 396 in betting rounds 2 and 3 respectively. Notably, all algorithms do not apply abstraction or embedding in the first round; instead, they rely on lossless isomorphism.

\begin{figure}[t]
\centering
\includegraphics[width=0.9\columnwidth]{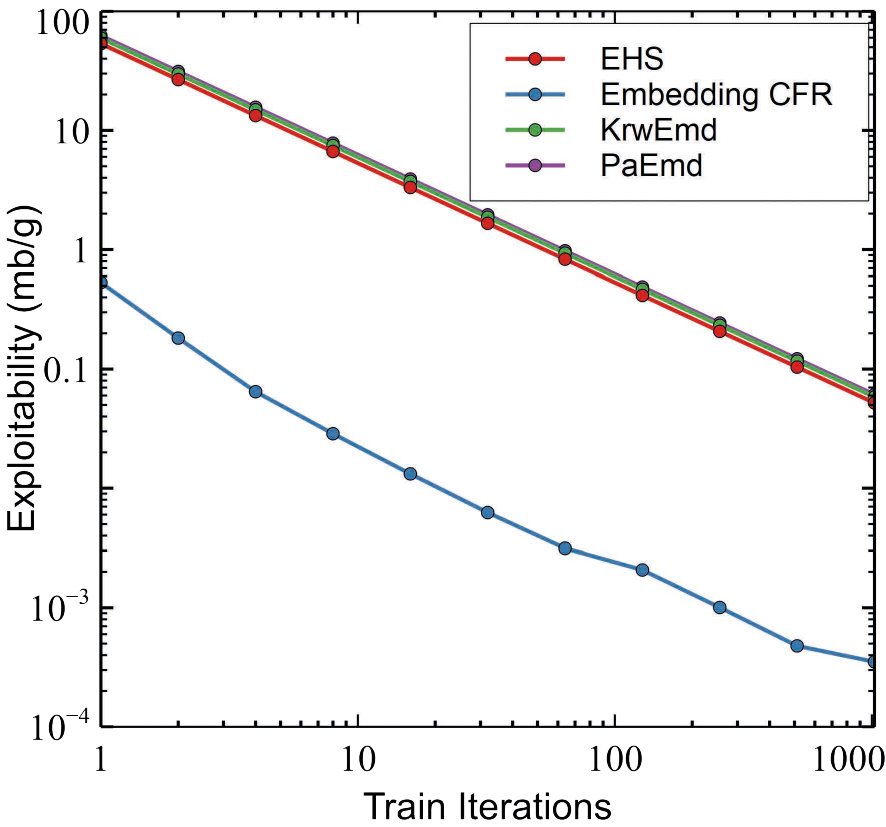} 
\caption{Exploitability convergence comparison of clustering-based algorithm (EHS, PaEmd, KrwEmd) vs. Embedding CFR algorithms in Numeral211 Hold'em.}
\label{fig:experiment1}
\end{figure}

We compare the convergence of exploitability among these algorithms, and the experimental results are presented in Figure \ref{fig:experiment1}. For the experiment, we sampled a number of hands equivalent to all possible hand combinations in a game, which is \(\binom{40}{2} \times \binom{38}{1} \times \binom{37}{1}\) (approximately \(10^6\)) as one iteration, and a total of 1024 iterations are conducted to compare the exploitability in various scenarios. The unit of exploitability is milli blind per game (mb/g). For each baseline comparison algorithm, we provide three parameter configurations. The figure shows the optimal experimental curve of the algorithm. More extensive experiments and details of the experimental equipment are provided in Appendix
\ref{apdx:experiment}
\negmedspace. 

It can be observed that the performance of abstraction methods based on hand clustering (EHS, PaEmd, KrwEmd) varies slightly, but they are at a comparable level. In contrast, the Embedding CFR algorithm demonstrates a significant performance boost in terms of exploitability reduction, with the quality of its strategy being notably superior to that of clustering-based algorithms. This result effectively validates the effectiveness of the Embedding CFR algorithm in the context of information set abstraction for poker games.

\section{Conclusion}
In conclusion, this paper presents the Embedding CFR algorithm, which innovatively introduces the embedding concept to address the pre-training process of information set abstraction in imperfect information games. This approach marks a significant breakthrough in the field, as it comprehensively outperforms clustering-based methods in our experimental settings.

\section*{Acknowledgments}

This work was supported by the National Science and Technology Major Project (Grant No. 2022ZD0116403) and the China Postdoctoral Science Foundation (Grant No. 2024M763533).

\bibliography{aaai2026}

\newpage
\appendix
\onecolumn 

\section{Rules of Referenced Imperfect Information Games} \label{apdx:game_rules} 

\subsection{Heads-Up Texas Hold'em} \label{apdx:heads_up_texas}  

Heads-Up Texas Hold'em is a two-player variant of Texas Hold'em, retaining the core mechanics of community cards and hand rankings while adapting to two-player dynamics. Below outlines the shared rules between its limit and no-limit variants, followed by their distinct betting structures.  

\begin{enumerate}
    \item \textbf{Betting Rounds:} Consists of four rounds—preflop, flop, turn, and river—where players can act sequentially after the corresponding cards are dealt.

    \item \textbf{Deck \& Hand Composition:}
    
        - Employs a standard 52-card deck (excluding Jokers), with each player privately dealt two private cards preflop.
  
        - Five community cards are revealed sequentially: three cards (the flop) are dealt first, followed by the turn (fourth card) and river (fifth card), forming the strongest possible 5-card hand in combination with the player's private cards.  
    \item \textbf{Hand Rankings:} Follow standard poker hierarchies, from Straight Flush (highest) to High Card (lowest), as detailed in Table \ref{tbl:texas_hand_rank}.  
    \item \textbf{Blinds:} Small Blind (SB) and Big Blind (BB) initiate the pot, with the SB posting half the BB amount.  

    \item \textbf{Action Order:} Preflop action starts with the SB; postflop (flop, turn, river) action starts with the BB.
    
    \item \textbf{AI Research Standard Setup:} Commonly used in academic studies, the SB is set to 100 chips, with each player starting with a 20,000-chip stack (200 BB effective stack depth).  
\end{enumerate}  

\begin{table}[H]
\centering
\begin{tabular}{cccp{5cm}c}
\hline
\textbf{Rank} & \textbf{Hand}         & \textbf{Prob.} & \multicolumn{1}{c}{\textbf{Description}}                                                                                     & \textbf{Example}                                      \\ \hline
1             & Straight Flush       & 0.0015\%       & Five cards of consecutive rank, all of the same suit. Ties broken by highest card in the sequence.                                                                 & $\texttt{A}_s\texttt{K}_s\texttt{Q}_s\texttt{J}_s\texttt{T}_s$ \\
2             & Four of a Kind       & 0.0240\%       & Four cards of the same rank. Ties broken by the rank of the four cards.                                                                                              & $\texttt{8}_s\texttt{8}_h\texttt{8}_d\texttt{8}_c\texttt{K}_h$ \\
3             & Full House           & 0.1441\%       & Three of a kind plus a pair. Ties broken by the rank of the three of a kind, then the pair.                                                                          & $\texttt{Q}_s\texttt{Q}_h\texttt{Q}_d\texttt{5}_c\texttt{5}_s$ \\
4             & Flush                & 0.1965\%       & Five cards of the same suit, not in sequence. Ties broken by comparing the highest card, then the next highest, etc.                                                     & $\texttt{9}_h\texttt{7}_h\texttt{6}_h\texttt{4}_h\texttt{2}_h$ \\
5             & Straight             & 0.3925\%       & Five cards of consecutive rank, not all the same suit. Ties broken by highest card in the sequence.                                                                      & $\texttt{K}_s\texttt{Q}_h\texttt{J}_c\texttt{T}_d\texttt{9}_h$ \\
6             & Three of a Kind      & 2.1128\%       & Three cards of the same rank. Ties broken by the rank of the three, then the highest remaining card, then the second remaining card.                                      & $\texttt{7}_s\texttt{7}_h\texttt{7}_d\texttt{K}_c\texttt{4}_h$ \\
7             & Two Pair             & 4.7539\%       & Two different pairs. Ties broken by the rank of the higher pair, then the lower pair, then the remaining card.                                                              & $\texttt{J}_s\texttt{J}_h\texttt{5}_c\texttt{5}_s\texttt{2}_h$ \\
8             & One Pair             & 42.2569\%      & Two cards of the same rank. Ties broken by the rank of the pair, then the highest remaining card, then the next highest, and so on.                                          & $\texttt{9}_s\texttt{9}_h\texttt{A}_c\texttt{K}_s\texttt{3}_h$ \\
9             & High Card            & 50.1177\%      & No pairs or better. Ties broken by comparing the highest card, then the next highest, etc.                                                                                & $\texttt{A}_s\texttt{K}_h\texttt{J}_c\texttt{8}_d\texttt{3}_h$ \\ \hline
\end{tabular}
\caption{Hand ranks of Texas Hold'em}
\label{tbl:texas_hand_rank}
\end{table}

\subsubsection{Heads-Up Limit Texas Hold'em (HULHE)} \label{apdx:limit_texas}  

The limit variant imposes fixed constraints on bet sizing and raise frequency, distinguishing it from no-limit play:  

\begin{enumerate}  
    \item \textbf{Bet Sizing:}  
        - Preflop bets and raises are capped at the big blind amount (e.g., equivalent to the BB value).  
        - Postflop (flop, turn, river) bets and raises double to twice the big blind, maintaining consistent increments.  
    \item \textbf{Raise Limits:} Each betting round allows a maximum of one initial bet plus three raises (e.g., bet → raise → re-raise → cap raise), preventing unbounded betting sequences.    
\end{enumerate}

\subsubsection{Heads-Up No-Limit Texas Hold'em (HUNL)} \label{apdx:nolimit_texas}  

The no-limit variant removes fixed bet constraints, enabling flexible sizing and strategic depth:  

\begin{enumerate}  
    \item \textbf{Bet Sizing:} Players may bet any amount from their remaining stack ($ge$ the big blind for preflop opens), with raises requiring a minimum size equal to the previous bet.  
    \item \textbf{Raise Flexibility:} No cap on the number of raises per round, allowing complex sequences (e.g., 3-bets, 4-bets) as long as each raise meets the minimum size requirement.    
    \item \textbf{All-In Mechanic:} A player can bet their entire stack at any round, locking the hand until showdown if called.  
\end{enumerate}

\subsection{Kuhn Poker}\label{apdx:kuhn}
Kuhn Poker is a simplified two-player poker variant introduced by Kuhn in 1950, serving as a fundamental model for studying imperfect information games in game theory and artificial intelligence research. Despite its simplicity, it captures key strategic elements such as bluffing and information asymmetry.
The game proceeds as follows:
\begin{enumerate}
\item \textbf{Ante:} Each player antes 1 chip into the pot.
\item \textbf{Deck:} The deck consists of three distinct cards: a Jack (\texttt{J}), a Queen (\texttt{Q}), and a King (\texttt{K}).
\item \textbf{Private Card:} Each player is dealt one private card face down.
\item \textbf{Betting Round:} The player who acts first (determined by a fixed order, typically alternating between hands) can choose to either check or bet 1 chip. The second player responds as follows:
\begin{itemize}
\item If the first player checked, the second player may check (ending the hand with a showdown) or bet 1 chip. If the second player bets, the first player must either call (leading to a showdown) or fold (conceding the pot to the second player).
\item If the first player bet, the second player must either call (leading to a showdown) or fold (conceding the pot to the first player).
\end{itemize}
\item \textbf{Showdown:} A showdown occurs if both players check, or if a bet is followed by a call. The player with the highest-ranked card (as specified in Table \ref{tbl:kuhn_hand_rank}) wins the pot.
\end{enumerate}

\begin{table}[H]
\centering
\begin{tabular}{ccc}
\hline
Rank & Hand & Probability \\ \hline
1 & \texttt{K} & 1/3 \\
2 & \texttt{Q} & 1/3 \\
3 & \texttt{J} & 1/3 \\ \hline
\end{tabular}
\caption{Hand ranks of Kuhn Poker}
\label{tbl:kuhn_hand_rank}
\end{table}

\label{apdx:simplifiedHUNL}

\subsection{Numeral211 Hold'em}\label{apdx:numeral211}
Numeral211 Hold'em is a poker variant more complex than Kuhn Poker but significantly simpler than HULHE. With diverse hand possibilities, it serves as an ideal testbed for studying hand abstraction tasks \cite{fu2024signal}.The game proceeds as follows:

\begin{enumerate}  
    \item \textbf{Ante:} Each player antes 5 chips into the pot at the start of the hand.  

    \item \textbf{Betting Rounds:} Consists of three rounds—preflop, flop, turn—where players can act sequentially after the corresponding cards are dealt.

    \item \textbf{Deck \& Hand Composition:}
    
        - Derived from a standard 54-card deck (excluding Jokers, Kings, Queens, and Jacks), resulting in 40 cards total. Suits include spades ($s$), hearts ($h$), clubs ($c$), and diamonds ($d$), with ranks 2 through 9, 10 (\texttt{T}), and Ace (\texttt{A}).
        
        - Each player receives one private card preflop. Two community cards are revealed sequentially, with one card in flop and one card in turn, respectively. The strongest 3-card hand is formed using the private card and community cards.   

    \item \textbf{Hand Rankings:} Follow hierarchies from Straight Flush (highest) to High Card (lowest), as detailed in Table \ref{tbl:numeral211_hand_rank}.  

    \item \textbf{Betting Options:} Players may fold, check, call, bet, or raise in all rounds (similar to HULHE). Each round allows a maximum of four total bets/raises, with fixed sizes: 10 chips in preflop and 20 chips in postflop rounds (flop and turn).  

    \item \textbf{Showdown:} If neither player folds, the highest-ranked hand wins the pot; ties result in an even split.  
\end{enumerate}  

\begin{table}[H]
\centering
\begin{tabular}{cccp{5cm}c}
\hline
\textbf{Rank} & \textbf{Hand} & \textbf{Prob.} & \multicolumn{1}{c}{\textbf{Description}} & \textbf{Example} \\ \hline
1 & Straight flush & 0.321\% & Three cards with consecutive rank and same suit. Ties are broken by highest card. & $\texttt{T}_s\texttt{9}_s\texttt{8}_s$ \\
2 & Three of a kind & 1.587\% & Three cards with the same rank. Ties are broken by the card's rank. & $\texttt{T}_s\texttt{T}_h\texttt{T}_c$ \\
3 & Straight & 4.347\% & Three cards with consecutive rank. Ties are broken by the highest card rank. & $\texttt{T}_s\texttt{9}_h\texttt{8}_c$ \\
4 & Flush & 15.799\% & Three cards with the same suit. Ties are broken by the highest card rank, then second highest, then third highest. & $\texttt{T}_s\texttt{8}_s\texttt{6}_s$ \\
5 & Pair & 34.065\% & Two cards with the same rank. Ties are broken by the pair's rank, then the third card's rank. & $\texttt{T}_s\texttt{T}_h\texttt{8}_c$ \\
6 & High card & 43.881\% & None of the above. Ties are broken by comparing the highest card, then second highest, then third highest. & $\texttt{T}_s\texttt{8}_h\texttt{6}_c$ \\ \hline
\end{tabular}
\caption{Hand ranks of Numeral211 Hold'em}
\label{tbl:numeral211_hand_rank}
\end{table}

\section{Pseudocode for Poker-Specific Embedding CFR}\label{apdx:pseudocode}

Algorithm~\ref{alg:embedding-cfr} presents the pseudocode for Poker-Specific Embedding CFR (e.g., tailored for Texas Hold'em), which leverages the unique properties of poker-like games as described in Section \ref{sec:partition_infoblock_chance}. Specifically, the progression of such games can be partitioned into a sequence of public tree nodes, which depict the game process identifiable to an observer. These nodes correspond to key game events, such as the chance player dealing community cards or other players taking actions like calling, betting, or folding. We denote the set of these public tree nodes as $V$.

All stochastic factors in the game can be determined prior to traversing the entire public tree—these factors refer to the private cards dealt to players and the community cards to be revealed by the final showdown stage. Traversing the public tree with these pre-determined cards is referred to as chance sampling.

Each hand of cards, when combined with a given public tree node $v$, uniquely defines an information set. It is straightforward to confirm that there exists a one-to-one correspondence between each public tree node $v$ and an info-block $J$. We use $J(v)$ to represent the info-block determined by node $v$.

\begin{algorithm} 
 \caption{Embedding CFR algorithm for two-player hold'em games}
 \label{alg:embedding-cfr}
 \begin{algorithmic}[1]
 \Require \(\tilde{R}_{advisor}(\cdot)\): A mapping from a public node \(v \in V\) to an embedded cumulative regret tensor of size \(|A(J(v))| \times m\)
 \Require \(\sigma_{advisor}(v)\): A mapping from a public node \(v \in V\) to an embedded immediate strategy tensor of size \(|A(J(v))| \times m\)
 \Require \(\bar{\sigma}_{advisor}(v)\): A mapping from a public node \(v \in V\) to an embedded average strategy tensor of size \(|A(J(v))| \times m\)
 \For{$t\leftarrow 1$ \textbf{to} $n$}
 \State Deal $l$ batches of final private card tensors \(\mathbf{h}_1, \mathbf{h}_2\) (for players 1 and player 2, respectively) and a final community card tensor \(\mathbf{b}\)
 \State \Call{Traverse}{$v^0, \mathbf{h}_1, \mathbf{h}_2, \mathbf{b}, \mathbf{1}_l, \mathbf{1}_l$} \Comment{$v^0$: public tree root; \;$\mathbf{1}_l$: a $l$-dimensional all-ones vector} 
 \EndFor
 
 \Function{Traverse}{$v, \mathbf{h}_1, \mathbf{h}_2, \mathbf{b}, \mathbf{p}_1, \mathbf{p}_2$}
 \If{$v$ is terminal public node}
 \State [$\mathbf{u}_1$, $\mathbf{u}_2$] \(\gets \Call{GetUtility}{v, \mathbf{h}_1, \mathbf{h}_2, \mathbf{b}}\) \Comment{\textsc{GetUtility}: Compute utilities from a public node and card tensors}
 \State \Return [$\mathbf{p}_2\circ \mathbf{u}_1$, $\mathbf{p}_1 \circ \mathbf{u}_2$] \Comment{``$\circ$" denotes element-wise multiplication}
 \EndIf
 \State $\mathbf{r}_1\gets\mathbf{0}_l$; \;$\mathbf{r}_2\gets\mathbf{0}_l$; \;$\mathbf{R}\gets\mathbf{0}_m$ \Comment{zero vectors of length $l$ or $m$}
 \ForAll{$i \in \{1,2\}$}
     \If{$v$ is player $i$'s public node}
     \State $\mathbf{x}_i \gets$ \Call{HandTensor}{$v, \mathbf{h}_i, \mathbf{b}$} \Comment{\textsc{HandTensor}: Truncate card tensors to $v$'s betting round; construct $l$-batch hand tensors (see Appendix~\ref{apdx:hand-tensor})}
     \State [\(\mathbf{\Phi}_{\cdot, \{j_1, \dots, j_l\}}\), \(\_\)] \(\gets \Call{HandEbdNet}{\mathbf{x}_i}\)
     \ForAll{$a \in A(J(v))$}
         \State $\sigma_{original} \gets (\mathbf{\Phi}_{\cdot, \{j_1, \dots, j_l\}})^\top \cdot \sigma_{advisor}(v)[a]$
         \State $\mathbf{p}_i' = \mathbf{p}_i\circ \sigma_{original}$; \;$\mathbf{p}_{3-i}' = \mathbf{p}_{3-i}$
         \State [$\mathbf{r}_1'$, $\mathbf{r}_2'$] \(\gets \Call{Traverse}{\textsc{NextNode}(v,a), \mathbf{h}_1, \mathbf{h}_2, \mathbf{b}, \mathbf{p}_1', \mathbf{p}_2'}\) \Comment{\textsc{NextNode}: Next public node after executing action $a$ at $v$}
         \State $\mathbf{r}_i \gets \mathbf{r}_i+\mathbf{r}_i'\circ \sigma_{original}$; \;$\mathbf{r}_{3-i} \gets \mathbf{r}_{3-i} + \mathbf{r}_{3-i}'$
        \State $\tilde{R}_{advisor}(v)[a] \gets \frac{T-1}{T}\tilde{R}_{advisor}(v)[a] +  \frac{1}{T} \mathbf{\Phi}_{\cdot, \{j_1, \dots, j_l\}} \cdot (\mathbf{r}_i'\circ \sigma_{original})$
        \State $\mathbf{R} \gets \mathbf{R} + \Call{TensorClamp}{\tilde{R}_{advisor}(v)[a], \varepsilon, \infty}$ \Comment{\textsc{TensorClamp}: values \(< \varepsilon \to \varepsilon\), otherwise unchanged}
     \EndFor
     \ForAll{$a \in A(J(v))$}
     \State $\sigma_{advisor}(v)[a] \gets \Call{TensorClamp}{\tilde{R}_{advisor}(v)[a], \varepsilon, \infty} \backslash \mathbf{R}$ \Comment{``$\backslash$" denotes element-wise division}
     \State $\bar{\sigma}_{advisor}(v)[a] \gets \frac{T-1}{T}\bar{\sigma}_{advisor}(v)[a] + \frac{1}{T}\sigma_{advisor}(v)[a]$
     \EndFor
     \EndIf
 \EndFor
 \State \Return [$\mathbf{r}_1$, $\mathbf{r}_2$]
 \EndFunction
 \Function{HandEbdNet}{$[\mathbf{x}_q]_{q\in Q}$}
    \Comment{$[\mathbf{x}_q]_{q\in Q}$: $|Q|$-batch hand tensor; $\mathbf{x}_q$ as described in Section~\ref{sec:embedding_matrices}}
    \State \Return [\([\mathbf{\Phi}_{\cdot, q}]_{q\in Q}\), \([\mathbf{y}_q]_{q\in Q}\)] \Comment{$\mathbf{\Phi}_{\cdot, q}$ and $\mathbf{y}_q$ as described in Section~\ref{sec:embedding_matrices}}
\EndFunction

\end{algorithmic}
\end{algorithm}

\section{Theoretical Proofs}\label{apdx:proofs}

\begin{lemma}[{\cite{lanctot2009monte}}, Lemma 7] \label{lemma:sq_inequality}
For any real numbers \(a\) and \(b\), define the positive truncation operator \(a_+ \triangleq \max(a, 0)\). The following inequality holds:
\[
\left[(a + b)_+\right]^2 \leq (a_+)^2 + 2(a_+) b + b^2.
\]
\end{lemma}

\begin{proof}
We analyze three exhaustive cases based on the signs of \(a\) and \(a + b\):

\textbf{Case 1: \(a \leq 0\)}  
Since \(a_+ = 0\), the right-hand side simplifies to \(b^2\). Meanwhile:
\[
\left[(a + b)_+\right]^2 \leq (b_+)^2 \leq b^2,
\]
confirming the inequality.

\textbf{Case 2: \(a \geq 0\) and \(b \geq -a\)}  
Here, \(a_+ = a\) and \((a + b)_+ = a + b\). Expanding both sides:
\[
\left[(a + b)_+\right]^2 = (a + b)^2 = a^2 + 2ab + b^2,
\]
which matches the right-hand side \((a_+)^2 + 2(a_+) b + b^2\).

\textbf{Case 3: \(a \geq 0\) and \(b \leq -a\)}  
We have \((a + b)_+ = 0\), and the right-hand side:
\[
(a_+)^2 + 2(a_+) b + b^2 = [(a_+) + b]^2 \geq 0 = \left[(a + b)_+\right]^2,
\]
satisfying the inequality.

Since all cases hold, the lemma is proven.
\end{proof}

\begin{lemma} \label{lemma:cfv_factorization}
For any infoset \(I \in \mathcal{I}_i\) and strategy profile \(\sigma\), the counterfactual value of player $i$ at $I$ satisfies:
\[v_i(\sigma, I) = \sum_{a \in A(I)} \sigma_i(I, a) \cdot v_i^{\sigma_{|I \to a}}(I)\]
\end{lemma}

\begin{proof}
By definition, the counterfactual value is:
\[v_i^\sigma(I) = \sum_{z \in Z_I} \pi_{-i}^\sigma(z[I]) \cdot \pi^\sigma(z[I], z) \cdot u_i(z)\]
Partition \(Z_I\) by actions at $I$: \(Z_I = \bigcup_{a \in A(I)} Z_I^a\), where \(Z_I^a\) contains terminal histories passing through $I$ via action $a$. This allows rewriting the summation as:
\[v_i^\sigma(I) = \sum_{a \in A(I)} \sum_{z \in Z_I^a} \pi_{-i}^\sigma(z[I]) \cdot \pi^\sigma(z[I], z) \cdot u_i(z)\]

For \(z \in Z_I^a\), decompose the transition probability: \(\pi^\sigma(z[I], z) = \sigma_i(I, a) \cdot \pi^\sigma(z[I] \cdot a, z)\). This decomposition holds because \(z[I] \in I\), and $I$ is an infoset where player $i$ acts. Substituting this gives:
\[v_i^\sigma(I) = \sum_{a \in A(I)} \sigma_i(I, a) \cdot \sum_{z \in Z_I^a} \pi_{-i}^\sigma(z[I]) \cdot \pi^\sigma(z[I] \cdot a, z) \cdot u_i(z)\]
For the modified strategy \(\sigma_{|I \to a}\) (which fixes action $a$ at $I$), note that \(\pi^{\sigma_{|I \to a}}(z[I], z) = \pi^\sigma(z[I] \cdot a, z)\) for all \(z \in Z_I^a\). By the definition of counterfactual value under \(\sigma_{|I \to a}\):
\[v_i^{\sigma_{|I \to a}}(I) = \sum_{z \in Z_I^a} \pi_{-i}^\sigma(z[I]) \cdot \pi^{\sigma_{|I \to a}}(z[I], z) \cdot u_i(z) = \sum_{z \in Z_I^a} \pi_{-i}^\sigma(z[I]) \cdot \pi^\sigma(z[I] \cdot a, z) \cdot u_i(z)\]
Substituting this equality for the inner summation yields:
\[v_i^\sigma(I) = \sum_{a \in A(I)} \sigma_i(I, a) \cdot v_i^{\sigma_{|I \to a}}(I)\]
This completes the proof.
\end{proof}

\begin{proof}[Proof of Proposition \ref{thm:regret_decrease}]
In the proposition's description of the approximate sampling scenario for Embedding CFR, consider iteration $T$ where player $i$ happens to select the strategy \(e_p\) for interaction across all information sets \(I \in J\). 

\textbf{Key Identity: Cross-Term Vanishing}

We first note that under this scenario, the following holds:
\begin{equation}
\sum_{a \in A(J)} \left(R^T(e_p, a)_+\right) r^{T+1}(e_p, a) = 0. \label{eq:thm_regret_decrease_eq1}
\end{equation}

If \(\sum_{a \in A(J)} \left(R^T(e_p, a)_+\right) = 0\), the conclusion holds trivially. This is because \(\left(R^T(e_p, a)_+\right)\) is non-negative, and a sum of non-negative terms equals zero only if each term individually is zero.

Then, when \(\sum_{a \in A(J)} \left(R^T(e_p, a)_+\right) > 0\):
\begin{align*}
&\sum_{a\in A(J)} (R^T(e_p, a)_+) r^{T+1}(e_p, a) \\
=& \sum_{a\in A(J)} (R^T(e_p, a)_+) \sum_{q=1}^n \Phi_{p,q}r^{T+1}(I_q,a) \\
=& \sum_{a\in A(J)} (R^T(e_p, a)_+) \sum_{q=1}^n \Phi_{p,q} \left(v_i^{\sigma^T}(I_q,a) - v_i^{\sigma^T}(I_q)\right) \\
=& \sum_{a\in A(J)} (R^T(e_p, a)_+) \sum_{q=1}^n \Phi_{p,q} \left(v_i^{\sigma^T}(I_q,a) - \sum_{a' \in A(J)}\sigma^T(I_q, a') v_i^{\sigma^T}(I_q,a')\right) \text{[By Lemma \ref{lemma:cfv_factorization}]} \\
=& \sum_{a\in A(J)} (R^T(e_p, a)_+) \sum_{q=1}^n \Phi_{p,q} \left(v_i^{\sigma^T}(I_q,a) - \sum_{a' \in A(J)}\sigma^T(e_p, a') v_i^{\sigma^T}(I_q,a')\right) \quad \text{[Since all infosets } I_q \in J \text{ use advisor } e_p\text{'s strategy]} \\
=& \sum_{a\in A(J)} (R^T(e_p, a)_+) \sum_{q=1}^n \Phi_{p,q} \left(v_i^{\sigma^T}(I_q,a) - \sum_{a' \in A(J)}\frac{R^T(e_p,a')_+}{\sum_{a''\in A(J)}(R^T(e_p,a'')_+)} v_i^{\sigma^T}(I_q,a')\right) \quad \text{[By regret matching ( Equation \eqref{eq:regret_matching_embedding})]} \\
=& \sum_{a\in A(J)}  \sum_{q=1}^n \Phi_{p,q} (R^T(e_p, a)_+) v_i^{\sigma^T}(I_q,a) - \sum_{q=1}^n\sum_{a' \in A(J)}\Phi_{p,q}(R^T(e_p,a')_+) v_i^{\sigma^T}(I_q,a') \\
=& 0
\end{align*}

\textbf{Bounding the Instantaneous Regret}

To characterize the regret reduction, we then bound \((r^{T+1}(e_p, a))^2\) using the Cauchy-Schwarz inequality:

\begin{align}(r^\sigma(e_p, a))^2&= \left( \mathbf{\Phi}_{p,\cdot} r^\sigma(\mathbf{J},a) \right)^2 \nonumber \\
&\leq \|\mathbf{\Phi}_{p,\cdot}\|_2^2 \cdot \|r^\sigma(\mathbf{J},a)\|_2^2 \quad \text{(by Cauchy-Schwarz inequality)} \nonumber \\
&= \|\mathbf{\Phi}_{p,\cdot}\|_2^2  \sum_{q=1}^n \left( v^{\sigma_{|I_q\to a}}(I_q) - v^\sigma(I_q) \right)^2 \nonumber \\
&\leq \|\mathbf{\Phi}_{p,\cdot}\|_2^2 \cdot \sum_{q=1}^n \Delta_J^2 \quad \text{(since \(|v^{\sigma_{|I_q\to a}}(I_q) - v^\sigma(I_q)| \leq \Delta_J\) for all $q$)} \nonumber \\
&= n \Delta_J^2 \cdot \|\mathbf{\Phi}_{p,\cdot}\|_2^2. \label{eq:specific_regret_bound}
\end{align}

\textbf{Bounding the Instantaneous Regret}

Define \(S^T = \sum_{a\in A(J)} \left(R^T(e_p, a)_+\right)^2\). We have:

\begin{align}
&S^{T+1} = \sum_{a\in A(J)} (R^{T+1}(e_p,a)_+)^2 \nonumber \\
=& \sum_{a\in A(J)} \left[ \left( \frac{T}{T+1} R^{T}(e_p,a) + \frac{1}{T+1} r^{T+1}(e_p,a) \right)_+ \right]^2 \nonumber \\
\leq& \frac{T^2}{(T+1)^2}  \sum_{a\in A(J)} \left(R^{T}(e_p,a)_+ \right)^2 + \frac{1}{T^2}\sum_{a\in A(J)}\left(  r^{T+1}(e_p,a) \right)^2 +  \frac{2T}{(T+1)^2}\overbrace{ \sum_{a\in A(J)}(R^{T}(e_p,a)_+) r^{T+1}(e_p,a)}^{\mathclap{\text{equals zero by Equation \eqref{eq:thm_regret_decrease_eq1}}}} \quad \text{(by Lemma \ref{lemma:sq_inequality})} \nonumber \\
=& \frac{T^2}{(T+1)^2} S^T + \frac{1}{(T+1)^2} \sum_{a\in A(J)} \left(r^{T+1}(e_p,a)\right)^2. \label{eq:thm_regret_decrease_eq2}
\end{align}

By Equation \eqref{eq:specific_regret_bound}, \(\sum_{a\in A(J)} \left(r^{T+1}(e_p,a)\right)^2 \leq |A(J)| \cdot n \Delta_J^2 \cdot \|\mathbf{\Phi}_{p,\cdot}\|_2^2 \triangleq C\), where $C$ is a constant. Inequality \eqref{eq:thm_regret_decrease_eq2} simplifies to:
\begin{equation}
S^{T+1} \leq \frac{T^2}{(T+1)^2}\cdot S^T + \frac{C}{(T+1)^2}
\end{equation}

We now analyze two scenarios for \(S^T\):

\textbf{1. Small \(S^T\): Controlled Polynomial Decay}\\
If \(S^T \leq \frac{C}{T}\), substituting into the inequality gives:
\[S^{T+1} \leq \frac{T^2}{(T+1)^2} \cdot \frac{C}{T} + \frac{C}{(T+1)^2} = \frac{T \cdot C + C}{(T+1)^2} = \frac{C}{T+1}.\]
This shows small \(S^T\) propagate to \(S^{T+1} \leq \frac{C}{T+1}\), with decay strictly controlled by the bound \(O\left(\frac{1}{T}\right)\).

\textbf{2. Large \(S^T\): Uncontrolled Polynomial Decay (Fractional Scaling)}\\
If \(S^T > \frac{C}{T}\), the dominant term drives a faster reduction:
\begin{align*}
    S^{T+1} &\leq \frac{T^2}{(T+1)^2} S^T + \frac{C}{(T+1)^2} \\
    &\leq \frac{T^2}{(T+1)^2} S^T + \frac{T\cdot S^T}{(T+1)^2}\\
    &=\frac{T}{T+1}S^T
\end{align*}

This means that when \(S^T > \frac{C}{T}\) holds, no explicit bound can be established for \(S^{T+1}\). Even so, it still guarantees a regret-decreasing process. Specifically, the magnitude of the decrease is proportional to \(S^T\) itself, yet decays with iteration.
\end{proof}

\section{Detailed Implementation of HandEbdNet in Numeral211 Hold'em}\label{apdx:handebdnet_implementation}

\subsection{The Construction of the Hand Tensor $\mathbf{x}_q$} \label{apdx:hand-tensor}

Numeral211 Hold'em, a 3-round poker game with 4 suits and 10 ranks, represents hands in the \(s\)-th round as \([4,10, s]\) image-like tensors (see Figure \ref{fig:hand_tensor}). Each channel is a \(10 \times s\) map encoding one suit's cards across rounds, facilitating independent analysis of straight flushes and flushes.

Preprocessing ensures isomorphic hands share identical tensors. Channels are ordered by suit importance: first by card count (more means a higher importance), with lexicographical order breaking ties (e.g., a suit containing an Ace in the first round and a 9 in the second round is considered more important than one with a 10 in the first round and a 0 in the second round). This uniquely distinguishes all suits in Numeral211 Hold'em.

The Figure \ref{fig:hand_tensor} illustrates examples of such tensor construction. Specifically, the hands \([\texttt{A}_c^p\texttt{T}_c^p, \texttt{9}_d, \texttt{2}_c]\) and \([\texttt{A}_h^p\texttt{T}_h^p, \texttt{9}_s, \texttt{2}_h]\) share an identical tensor representation, and their predecessors in the second round are also identical.

\begin{figure}[H]
\centering
\includegraphics[width=0.4\columnwidth]{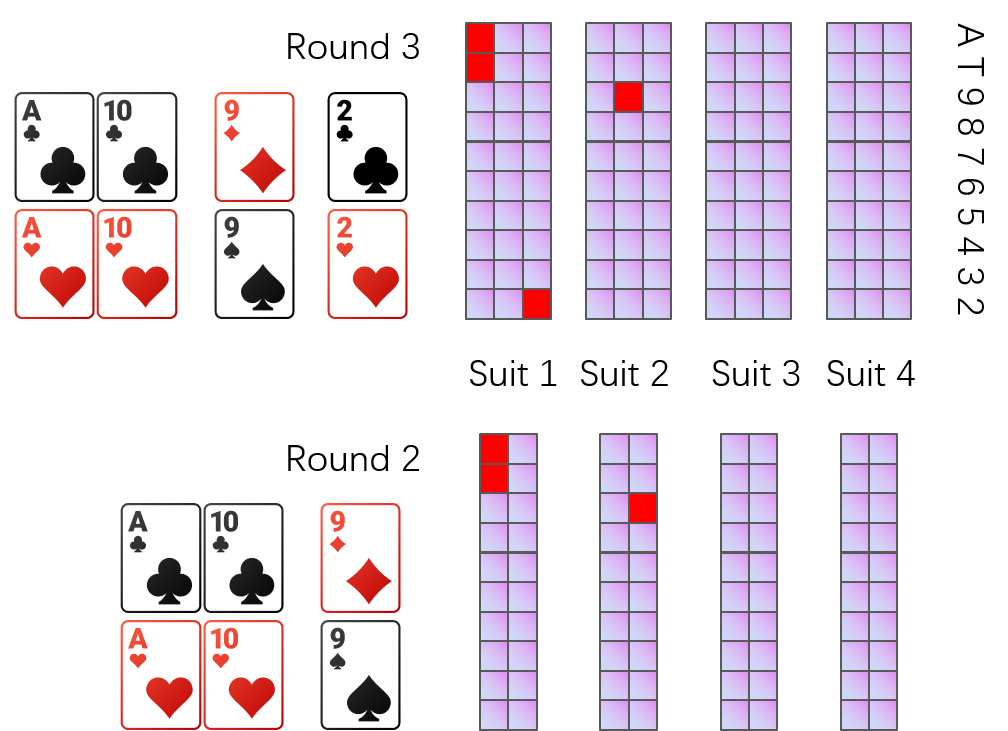} 
\caption{Hand Tensor Examples in Numeral211 Hold'em.}
\label{fig:hand_tensor}
\end{figure}

This approach is equally applicable to more complex Texas Hold'em games.

\subsection{The Construction of the Hand Strength Tensor $\mathbf{y}_q$}

Let's formalize the construction process of the hand strength tensor $\mathbf{y}$ of a hand $x$. We largely draw on the ideas by \citet{fu2024signal} to record the strengths of $x$ and its predecessors in their respective rounds.

Suppose a poker game consists of a total of $S$ rounds ($S=3$ in Numeral211 Hold'em). We construct the hand strength tensor for a hand $x$ at round $s$ through the following systematic process:

\textbf{For the terminal round \(S\)}: Hands \(x\) in poker games naturally form a total order at terminal rounds. We directly compare its strength against all other valid opponent hands \(x'\) that share the same community cards at this round. This comparison yields a 3-length outcome vector \(\mathbf{w} = [w_l, w_d, w_w]\), where:
\begin{itemize}
    \item \(w_l = \frac{\text{Number of } x' \text{ where } x \text{ loses to } x'}{\text{Total number of } x'}\)
    \item \(w_d = \frac{\text{Number of } x' \text{ where } x \text{ draws with } x'}{\text{Total number of } x'}\)
    \item \(w_w = \frac{\text{Number of } x' \text{ where } x \text{ wins against } x'}{\text{Total number of } x'}\)
\end{itemize}
These values are inherently normalized such that \(w_l + w_d + w_w = 1\).

In Numeral211 Hold'em, for example, during the terminal round, each player holds one private card, and there is one community card on the table. For a specific hand \(x\) (composed of the player's private card and the community card), the comparison set includes all other possible hands formed by pairing each of the remaining 38 private cards with the same community card (i.e., \(x'\)). The vector \(\mathbf{w}\) is thus calculated based on how \(x\) performs against these 38 potential opponent hands.

\textbf{For a non-terminal round \(s\)}: Since direct strength comparison between hands $x$ is not feasible at non-terminal rounds, we employ a rollout process: each hand $x$ at round \(s\) is rolled out to the terminal round \(S\). For each rollout scenario $x\sqsubset x_k$ in $S$, we obtain a 3-length outcome vector \(\mathbf{w}_k\) following the terminal round comparison method. The final outcome vector for round \(s\) is calculated as a weighted sum \(\mathbf{w} = \sum_{k=1}^K \alpha_k \mathbf{w}_k\), where \(\alpha_k\) represents the probability weight of the \(k\)-th rollout path, and \(\sum_{k=1}^K \alpha_k = 1\).

For a hand \(x\) evaluated at round \(s\), we compute the 3-length outcome vectors for each round \(1, 2, \ldots, s\) using the above methods applied to the predecessors of $x$. The hand strength tensor \(\mathbf{y}_q \in \mathbb{R}^{s \times 3}\) is then formed by stacking these vectors vertically, where each row corresponds to the outcome vector of a specific round, resulting in a 2-dimensional tensor with dimensions \([s, 3]\). This tensor provides a comprehensive strength profile across all relevant game rounds.

\subsection{Structure of the Hand Feature Encoder}

The hand feature encoder (HFE) is a simple structure. Since we model a hand as an image-like tensor, the HandEbdNet--essentially a regression task mapping \(\mathbf{x}_q\) to \(\mathbf{y}_q\)--adopts a CNN architecture to identify hand patterns (see Figure \ref{fig:hand_feature_encoder}). Logically, HFE comprises a convolutional layer followed by a fully connected layer.

\begin{figure}[H]
\centering
\includegraphics[width=0.9\columnwidth]{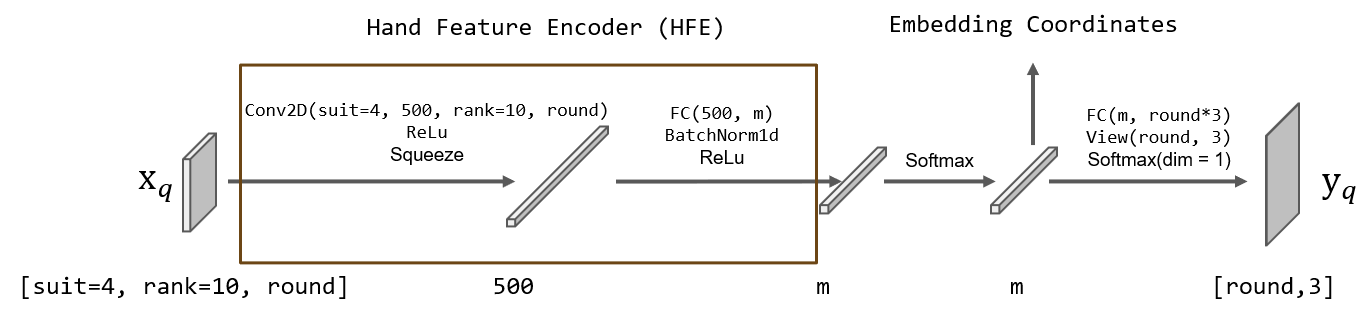} 
\caption{Structure of the HandEbdNet.}
\label{fig:hand_feature_encoder}
\end{figure}

Notably, each convolution kernel matches the size of a single channel, mapping one channel's data to a single value. This enables kernels to recognize suit-specific patterns (e.g., straight flushes, flushes). After ReLU activation and squeezing, a 500-length intermediate vector (hyperparameter) is produced. The subsequent fully connected layer then generates the target m-length vector, while capturing cross-suit hand strengths like pairs and straights.

\subsection{Training and Inference}

We train a separate network for each round, considering only Round 2 and Round 3. As shown in Figure \ref{fig:hand_feature_encoder}, we need to specify the ``round" parameter. The training data consists of all possible hand combinations in the current round: for Round 2, the number of combinations is \(\binom{40}{2} \times \binom{38}{1} = 780 \times 38 = 29,\!640\); for Round 3, it is \(\binom{40}{2} \times \binom{38}{1} \times \binom{37}{1} = 1,\!096,\!680\). Such data volumes are easily traversable and trainable on a home PC. Since the same data is used during inference, we do not need to consider issues of network robustness or overfitting--this is essentially a simple regression task, which is why a basic network structure suffices to solve this task.

\section{Experiment}\label{apdx:experiment}

\begin{figure}[t]
\centering
\includegraphics[width=0.9\columnwidth]{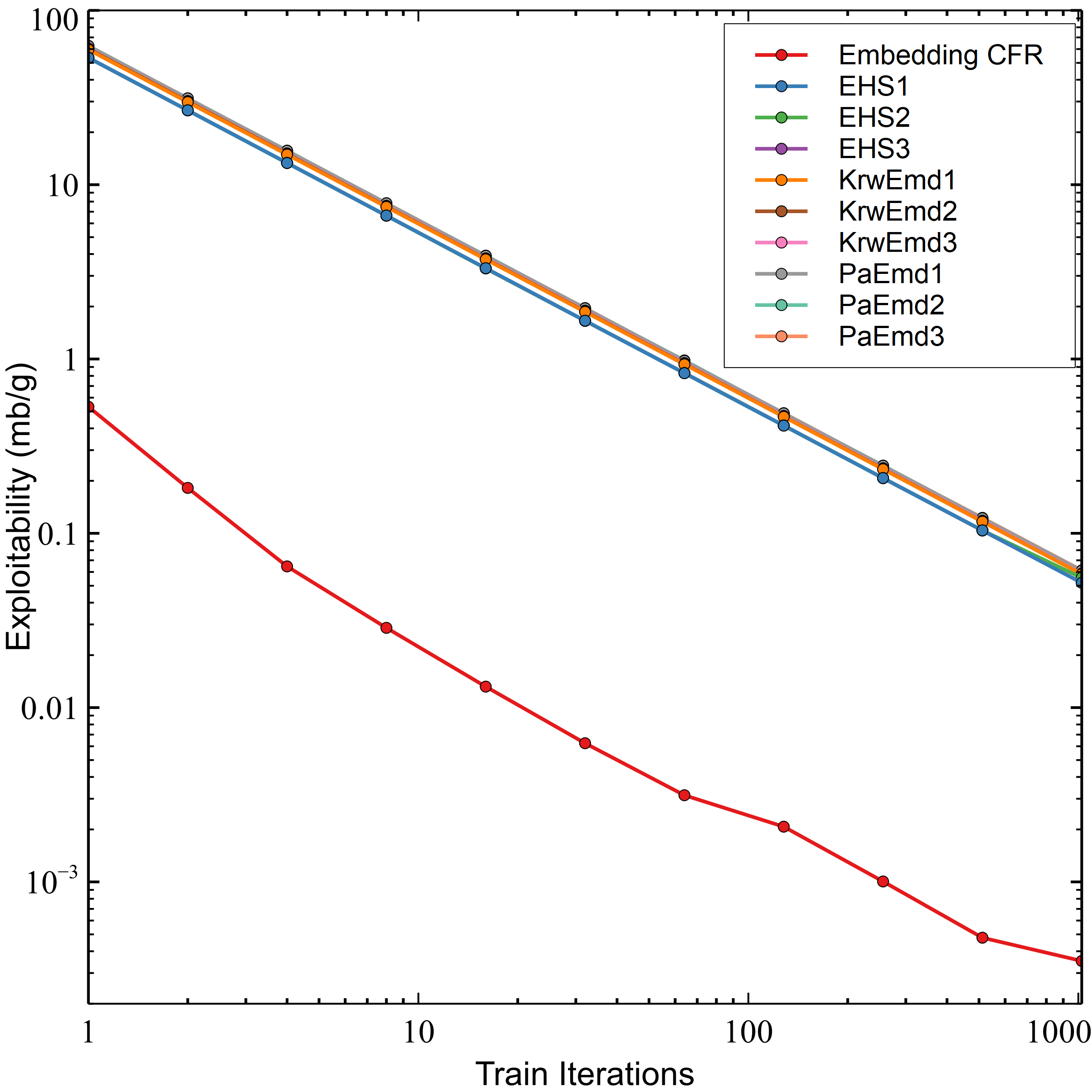} 
\caption{Complete Experimental Data}
\label{fig:Complete_experiment}
\end{figure}

For the EHS and PaEmd algorithms, there are no parameters to set; the only factor that can lead to differences is the initial clustering algorithm. For each algorithm, we generate three initial clusterings and iteratively compute the final clustering results.

For the KrwEmd algorithm, it requires setting an importance hyperparameter, which simply refers to the significance of the winning rate distance in each round. We select three sets of parameters: KrwEmd1, where the winning rate differences in late game are more important; KrwEmd2, where the winning rate differences in early game are more important; and KrwEmd3, where the winning rate differences in all rounds are equally important. One clustering result is generated for each set.

The complete experimental data are shown in the Figure \ref{fig:Complete_experiment}. It can be observed that in long-term experiments, the differences in exploitability among the various clustering-based methods are insignificant, yet there exists an essential gap when compared with EmbeddingCFR in terms of exploitability.

Our experiments were conducted on a device equipped with 754GB of memory, an Intel(R) Xeon(R) Gold 6240R CPU (2.40GHz base frequency, up to 4.00GHz, 96 threads), and 7 NVIDIA TITAN RTX GPUs (each with 24GB memory). Each individual experiment takes approximately 34.2 days (with 32 threads).

\end{document}